%% file: ijcai26.tex
\documentclass{article}
\usepackage{ijcai26}

\usepackage{times}
\usepackage{soul}
\usepackage{url}
\usepackage[hidelinks]{hyperref}
\usepackage[utf8]{inputenc}
\usepackage[small]{caption}
\usepackage{graphicx}
\usepackage{amsmath}
\usepackage{amsthm}
\usepackage{booktabs}
\usepackage{algorithm}
\usepackage{algorithmic}
\usepackage[switch]{lineno}

\newcommand{\vpara}[1]{\vspace{0.05in}\noindent\textbf{#1 }}
\newcommand{\vpait}[1]{\vspace{0.05in}\noindent\textit{#1 }}

\usepackage{makecell}
\usepackage{multirow}
\usepackage{booktabs}
\usepackage{amsmath}
\usepackage{amssymb}

\newtheorem{proposition}{Proposition}

\title{UniSAGE: Unifying Static and Dynamic Attributes with Hyper-Structure}

\author{
Taoran Fang$^1$
\and
Yan Deng$^1$\and
Chunping Wang$^2$\and
Yang Wang$^2$\and
Lei Chen$^2$ \And
Yang Yang$^{1}$\thanks{Corresponding author.}\\
\affiliations
$^1$Zhejiang University\\
$^2$FinVolution Group\\
\emails
\{fangtr, dyan, yangya\}@zju.edu.cn,
\{wangchunping02, wangyang09, chenlei04\}@xinye.com
}

\begin{document}

\maketitle

\input{00_abstract}
\input{01_introduction}

\input{02_related_work}
\input{03_preliminary}

\input{04_methodology}
\input{05_experiment}

\input{06_conclusion}

\bibliographystyle{named}
\bibliography{reference}

\clearpage
\input{07_appendix}

\end{document}

%% file: 00_abstract.tex
\begin{abstract}

With the rapid growth of digital data, real-world applications increasingly involve hierarchical information that combines static attributes with dynamic records.
Modeling such heterogeneous data in a unified and generalizable manner remains challenging.
Existing approaches often rely on extensive manual design, are tightly coupled to specific data schemas, and typically process static and dynamic attributes in isolation, thereby overlooking their implicit interactions.
We propose UniSAGE, a unified framework for modeling data with both static and dynamic attributes.
UniSAGE constructs a global attribute graph that represents hierarchical and temporal relationships in a unified structure.
To ensure representational consistency, it introduces two orthogonal parameter subspaces that jointly support static aggregation and dynamic reasoning within a shared semantic space.
Building on these unified representations, UniSAGE further enables task-specific interaction between static and dynamic attributes via a lightweight hyper-structure mechanism.
UniSAGE is fully automated, robust to evolving data schemas, and capable of capturing complex cross-attribute dependencies.
Extensive experiments on multiple public benchmarks and a real-world financial behavior dataset demonstrate that UniSAGE consistently outperforms existing methods, achieving performance improvements of over 10\% on several tasks.
Our code is available at \url{https://github.com/zjunet/UniSAGE}.


\end{abstract}

%% file: 01_introduction.tex
\section{Introduction}

In the era of big data, modern information systems generate and collect data that are increasingly diverse, heterogeneous, and large-scale.
From financial transactions and e-commerce activities to healthcare records and social behaviors, such data are used to model a wide range of real-world entities.
As data sources continue to grow in both complexity and volume, efficiently extracting meaningful representations to support downstream tasks remains a fundamental challenge~\cite{wang2018review,xiao2022survey}.

This paper focuses on feature extraction from a class of complex data that we refer to as hierarchical data with both static and dynamic components.
Such data typically comprise static attributes (\emph{e.g.}, age, gender, product category) together with dynamic records (\emph{e.g.}, user actions, transaction sequences, event logs).
In practice, these heterogeneous attributes are often organized in semi-structured formats such as JSON, where entities exhibit hierarchical schemas with nested attributes and timestamped logs.
This data paradigm is ubiquitous across real-world applications, including recommendation systems~\cite{middleton2004ontological,hossain2024socialrec}, healthcare records~\cite{davis2014patient,agniel2018biases}, and financial transaction monitoring~\cite{jayathilake2012towards,sheble2009transaction}.
Effectively modeling such data is critical for downstream tasks such as classification~\cite{oza2008classifier}, risk assessment~\cite{gritzalis2018exiting}, and recommendation~\cite{javed2021review}.

However, effectively modeling hierarchical data with static and dynamic attributes remains challenging.
First, static and dynamic attributes are often modeled separately, leading to incompatible representation spaces and fragmented semantic reasoning.
As a result, capturing a holistic entity representation typically requires extensive manual model engineering, which is costly and difficult to scale.
Second, many existing approaches are tightly coupled to fixed data schemas, limiting their adaptability.
In real-world industrial systems, data schemas frequently evolve as attributes are added or modified, yet most models require retraining from scratch to accommodate such changes.
Third, capturing interactions between static and dynamic attributes is inherently difficult.
In many applications, these two attribute types are strongly interdependent.
For example, in financial lending, the combination of recent loan activities (dynamic) and age-related information (static) often jointly determines default risk.
Nevertheless, most existing models process static and dynamic attributes in isolation, missing crucial cross-type correlations that are essential for accurate decision-making.

To address these challenges, we propose \textbf{Uni}fying \textbf{S}tatic and dynamic \textbf{A}ttributes with \textbf{G}raph \textbf{E}nhancement (UniSAGE), a unified framework that models static and dynamic attributes within a shared representation space while explicitly capturing their interactions.
UniSAGE first constructs a global attribute graph that reflects the hierarchical structure of the original data.
Static attributes and timestamped dynamic records are represented uniformly as nodes, with static links encoding hierarchical relations and dynamic links modeling temporal dependencies.
This unified graph provides a structured modeling space in which information can be propagated coherently across attribute types.
The representation of an entity is obtained by aggregating information from all nodes into a root node.
To ensure representational consistency, UniSAGE introduces a joint parameterization scheme with two orthogonal subspaces for static aggregation and dynamic reasoning.
Orthogonality prevents mutual interference between the two modeling paradigms, allowing static and dynamic attributes to be processed independently while remaining embedded in a unified semantic space.
As a result, UniSAGE no longer distinguishes static and dynamic attributes at the representation level, but only through their respective propagation paths.
This design enables automatic handling of heterogeneous attributes and provides strong generalizability to evolving data schemas.
Building on these unified representations, UniSAGE further constructs task-specific hyper-structures to capture higher-order interactions between selected static and dynamic attributes.
These hyper-structures introduce specialized reasoning paths tailored to downstream objectives.
To avoid the high computational cost of explicit hyper-structure construction, we propose \textbf{S}elective \textbf{S}emantic \textbf{Agg}regation (SSAgg), a lightweight mechanism that selectively aggregates task-relevant information.
We theoretically show that SSAgg can simulate arbitrary hyper-structures, and empirically demonstrate its effectiveness across diverse tasks.
Overall, our contributions are summarized as follows:
\begin{itemize}
\item We propose UniSAGE, a unified framework that models static and dynamic attributes in a shared representation space, enabling automated and generalizable learning over hierarchical data.
\item We introduce SSAgg, a theoretically grounded and computationally efficient mechanism for capturing task-specific interactions across static and dynamic attributes.
\item Extensive experiments on public benchmarks and a real-world financial dataset demonstrate that UniSAGE consistently outperforms existing methods, achieving improvements of over 10\% on several tasks.
\end{itemize}

%% file: 02_related_work.tex
\section{Related Work}

\vpara{Graph-Based Encoding.}
Transforming complex or semi-structured data into graph representations is a common paradigm in data mining and representation learning~\cite{jeon2013making,patil2018survey}.
Such graph-based encodings have been widely applied in recommendation systems~\cite{wang2021graph}, healthcare~\cite{hassanain2022evaluating,arch2017graph}, and finance~\cite{zehra2021financial,zhang2023dynamic,onnela2003asset}, where rich dependencies among heterogeneous components must be modeled explicitly.
Graph neural networks (GNNs) are a natural choice for learning over such graph-structured data~\cite{wu2020comprehensive,zhou2020graph}.
Static GNN models~\cite{kipf2016semi,hamilton2017inductive,velivckovic2017graph} focus on learning representations from fixed graph structures and are effective in capturing hierarchical or relational patterns.
However, they are not designed to handle temporal evolution.
To address this limitation, dynamic GNNs~\cite{skarding2021foundations,feng2024comprehensive,yang2024dynamic} extend static models by incorporating temporal information and evolving graph structures.
Despite their success, most dynamic GNNs assume that temporal information is aligned across nodes or edges, or that dynamic interactions are defined over a shared timeline.
This assumption is often violated in real-world scenarios involving heterogeneous attributes and asynchronous records, where different attributes may have varying numbers of events and non-aligned timestamps.
Moreover, existing graph-based methods typically treat static and dynamic information using different modeling pipelines, which complicates representation alignment and parameter sharing.
In contrast, UniSAGE adopts an attribute-level graph formulation in which both static attributes and dynamic records are treated as nodes.
By jointly encoding hierarchical relations among attributes and temporal dependencies among records, this formulation avoids the need for strict temporal alignment and enables unified processing of heterogeneous information.
In our experiments, we include representative static and dynamic GNN models as baselines to evaluate the effectiveness of this unified design.

\vpara{Relational Deep Learning.}
Relational deep learning (RDL)~\cite{fey2023relational,dwivedi2025relational,robinson2024relbench,chen2025relgnn} studies learning from relational data stored in tabular databases, where entities are described by both static attributes and dynamic records.
RDL methods typically construct graphs by treating each table row (i.e., record) as a node and connecting nodes according to foreign-key relationships.
This formulation has shown strong performance on benchmark datasets such as RelBench~\cite{fey2023relational}.
While RDL and UniSAGE address similar data modalities, their modeling strategies differ fundamentally.
RDL operates at the record level and relies on schema-defined relations between rows, which limits its ability to explicitly capture hierarchical dependencies among attributes.
In contrast, UniSAGE adopts an attribute-centric modeling paradigm, where both static and dynamic attributes are explicitly represented and organized according to hierarchical and temporal relations.
This design enables finer-grained reasoning over heterogeneous attributes and naturally supports recursive, multi-level representations.
We evaluate UniSAGE on the RelBench benchmark and directly compare it with representative RDL methods~\cite{fey2023relational}.
The results demonstrate that UniSAGE is more effective in modeling complex hierarchical and dynamic structures, especially in settings where attributes exhibit heterogeneous schemas and asynchronous records.

%% file: 03_preliminary.tex
\section{Preliminary}

\vpara{Categories of Data Attributes.}
In real-world applications, complex data used for modeling is commonly represented in tabular or semi-structured formats such as JSON.
Attributes in such data can be broadly categorized as either \textit{static} or \textit{dynamic}, depending on whether their values evolve over time.

\vpait{Static Attribute.}
A static attribute refers to a property of an entity whose value remains constant or changes negligibly over a sufficiently long period.
For instance, in a credit loan system, user characteristics such as gender or place of birth are typically treated as static attributes.
The value of a static attribute $s$, denoted as $x_s$, depends only on the associated entity $e$ and is defined as
\begin{align}
x_s(e) = f_s(e),
\end{align}
where $f_s(\cdot)$ is the mapping function corresponding to the static attribute $s$.

\vpait{Dynamic Attribute.}
Dynamic attributes describe properties whose values change over time or accumulate records as time progresses.
For example, a user’s residential address may change over time, and a user’s borrowing history grows as new loan records are added.
The value of a dynamic attribute $d$ at time $t$, denoted as $x_d^t$, depends on both the entity $e$ and the timestamp $t$:
\begin{align}
x_d^t(e) = f_d(e, t).
\end{align}
We represent a dynamic attribute as a time-indexed collection of records,
\begin{align}
x_d(e) = \{(x_d^t(e), t) \mid t \in [1, T]\},
\end{align}
where $f_d(\cdot)$ is the mapping function associated with attribute $d$, and $T$ denotes the number of observed records.
In practice, dynamic attributes may vary across entities in both sequence length and timestamp distribution.
For example, different users typically have different numbers of loan records, and their timestamps are rarely aligned.

\vpara{Overall Objective.}
Given an entity $e$ with its associated static and dynamic attributes, our goal is to learn a representation $h_e$ through a function $\phi$:
\begin{align}
\label{eq:target}
h_e = \phi\big(\{x_s(e) \mid s \in \mathcal{S}\}, \{x_d(e) \mid d \in \mathcal{D}\}\big),
\end{align}
where $\mathcal{S}$ and $\mathcal{D}$ denote the sets of static and dynamic attributes, respectively.
The learned representation $h_e$ serves as the input to downstream tasks such as classification, regression, or user modeling.

%% file: 04_methodology.tex
\section{Methodology}

\begin{figure*}[tbp]
\centering
\includegraphics[width=0.9\textwidth]{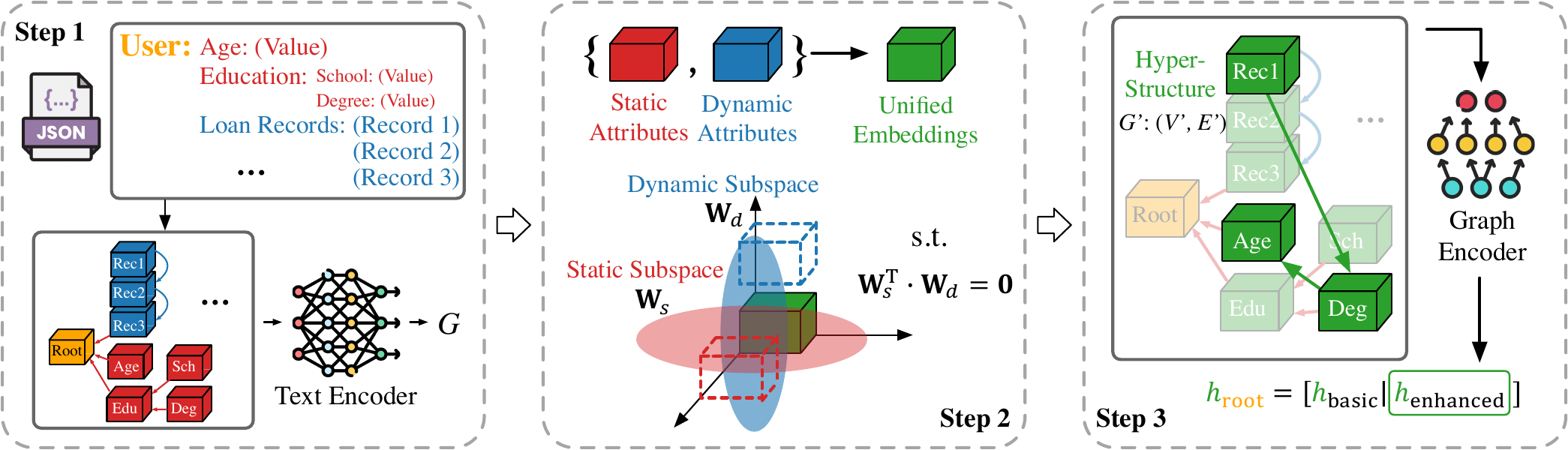}
\caption{Overview of UniSAGE.
\textbf{Step 1:} Construct a global attribute graph from JSON, with nodes as attributes/records and edges encoding hierarchical (static) and temporal (dynamic) relations.
\textbf{Step 2:} Learn unified representations via orthogonal subspaces for non-interfering static aggregation and dynamic reasoning.
\textbf{Step 3:} Capture task-specific cross-type interactions through a lightweight hyper-structure mechanism (SSAgg) to enhance downstream prediction.}
\label{fig:main_figure}
\end{figure*}

\subsection{Hierarchical Graph Construction}
\label{sec:hgc}
We represent each entity instance as a directed hierarchical graph $G$ that unifies both static attributes and dynamic records.
The graph is organized as a rooted tree, where edges point from sub-attributes (or earlier records) to their parents (or later records), and the root aggregates all information into an entity embedding.

\vpara{Node Generation.}
Each attribute is instantiated as a node following the hierarchical schema.
Given an attribute $a$, we recursively create a node $v_a$ and expand its sub-attributes $\operatorname{Sub}(a)$ until reaching leaf attributes:
\begin{align}
g(a)=
\begin{cases}
v_a, & \text{if } \operatorname{Sub}(a)=\emptyset,\\
v_a \rightarrow \{g(a') \mid a' \in \operatorname{Sub}(a)\}, & \text{otherwise}.
\end{cases}
\end{align}
For a dynamic attribute $d$ with $T$ timestamped records, we treat each record as an individual node $\{v_{d_t}\}_{t=1}^T$ (and recursively expand its sub-attributes if present):
\begin{align}
g(d)=
\begin{cases}
\{v_{d_t}\}_{t=1}^T, & \text{if } \operatorname{Sub}(d_t)=\emptyset\ \forall t,\\
\{v_{d_t}\}_{t=1}^T \rightarrow \{g(d_t)\}_{t=1}^T, & \text{otherwise}.
\end{cases}
\end{align}

\vpara{Static Links.}
Static links encode the schema hierarchy: for each parent attribute $a_j$ and its (static) sub-attribute $a_i \in \operatorname{Sub}(a_j)$, we add a directed edge $v_{a_i}\!\to\! v_{a_j}$.
We define the static adjacency matrix $\mathbf{A}_s\in\mathbb{R}^{N\times N}$ as
\begin{align}
\mathbf{A}_s(i,j)=
\begin{cases}
1, & \text{if } a_i \in \operatorname{Sub}(a_j),\\
0, & \text{otherwise}.
\end{cases}
\end{align}
When a sub-attribute is dynamic (i.e., it corresponds to a record sequence $\{d_t\}_{t=1}^T$), connecting the parent to all records would introduce redundant paths and inflate computation.
Instead, we connect the parent only to the most recent record node $d_T$, which preserves the latest state while keeping the hierarchy concise:
\begin{align}
\mathbf{A}_s(d_T, a_j) &= 1,\\
\mathbf{A}_s(d_t, a_j) &= 0,\quad t<T,\quad \text{for } \{d_t\}_{t=1}^{T}\in \operatorname{Sub}(a_j).
\end{align}

\vpara{Dynamic Links.}
Dynamic links encode temporal order within each dynamic attribute.
For consecutive records of the same attribute, we add directed edges from earlier to later records:
\begin{align}
\mathbf{A}_d(i,j)=
\begin{cases}
1, & \text{if } a_i=d_t \text{ and } a_j=d_{t+1},\\
0, & \text{otherwise}.
\end{cases}
\end{align}

Finally, we add a root node $r$ and connect each top-level attribute node to $r$ with a directed edge, forming a rooted tree.
The entity representation is obtained from the final embedding of the root node $r$.

\subsection{Orthogonal Transformation Design}
\label{sec:otd}
UniSAGE performs bottom-up propagation on $G$ to compute the root representation.
A key challenge is that static aggregation (hierarchical composition) and dynamic reasoning (temporal modeling) typically rely on different operators (\emph{e.g.}, GNN/MLP vs. RNN/GRU), which can distort semantics and hinder parameter sharing even when inputs are encoded in the same embedding space.
To address this, we introduce two \emph{orthogonal} parameter subspaces: $\mathbf{W}_s$ for static aggregation and $\mathbf{W}_d$ for dynamic reasoning.
Orthogonality encourages the two pathways to be non-interfering while keeping representations in a shared semantic space.
Formally, we first map node features into a high-dimensional space $\mathbf{H}\in\mathbb{R}^{N\times k}$.
We then constrain parameters for static and dynamic transformations to lie in the spans of $\mathbf{W}_s$ and $\mathbf{W}_d$, respectively, and enforce
\begin{align}
\theta_{s}\in \text{span}(\mathbf{W}_{s}),\ \theta_{d}\in \text{span}(\mathbf{W}_{d})
\quad \text{s.t. } \mathbf{W}_{s}^\top \mathbf{W}_{d} = \mathbf{0}.
\end{align}
Here $\text{span}(\mathbf{W})$ denotes the subspace spanned by columns of $\mathbf{W}$.
This constraint reduces cross-talk between static and dynamic transformations, stabilizing training and enabling consistent semantics under different propagation paths.

\subsection{Hyper-Structure Generation}
\label{sec:hsg}
The global graph $G$ captures hierarchical and temporal relations but does not explicitly model \emph{task-specific} interactions between static and dynamic attributes.
However, many decisions depend on cross-type combinations (e.g., recent behaviors conditioned on demographics).
To capture such interactions, we define a task-specific hyper-structure as a subgraph $G'=(\mathcal{V}',\mathcal{E}')$ induced from $G=(\mathcal{V},\mathcal{E})$, where $\mathcal{V}'\subseteq \mathcal{V}$ contains task-relevant nodes and $\mathcal{E}'$ connects them to form specialized reasoning paths.
We then compute a hyper-structure representation via a graph encoder $f$:
\begin{align}
\label{formula:hsg}
h_{G^\prime}=f(\mathcal{V}^\prime,\mathcal{E}^\prime),\ \text{s.t.}\ \mathcal{V}^\prime \subseteq \mathcal{V}.
\end{align}
In practice, we instantiate $f$ with a GNN due to its strong relational reasoning ability~\cite{xu2020neural,xu2019can,Yang2022GraphNN,Yang2023HowGN,Dudzik2022GraphNN}.
Finally, $h_{G'}$ is combined with the root representation $h_r$ for downstream prediction.

\subsection{Practical Implementation}
\label{sec:pi}
We now describe the end-to-end implementation (Figure~\ref{fig:main_figure}), corresponding to Sections~\ref{sec:hgc}--\ref{sec:hsg}.

\vpara{Step 1: Global graph and node features.}
We construct $G$ using $\mathbf{A}_s$ and $\mathbf{A}_d$ as in Section~\ref{sec:hgc}.
Each node is associated with a text input that describes its attribute/value.
Specifically, for leaf nodes we use ``attribute name: value'';
for intermediate attribute nodes we use ``attribute name: [attribute flag]'';
and for the root node we use a special marker ``[root flag]''.
We encode each node text with a pretrained text encoder:
\begin{align}
x_{\text{raw}}(v)&=
\begin{cases}
\text{``attribute name: value''},\ \text{if } v \text{ is a leaf node},\\
\text{``[root flag]''}, \ \text{if } v \text{ is the root node},\\
\text{``attribute name: [attribute flag]''}, \ \text{otherwise},
\end{cases} \nonumber\\
x(v) &= \text{TextEncoder}\big(x_{\text{raw}}(v)\big),
\end{align}
yielding the node feature matrix $\mathbf{X}\in\mathbb{R}^{N\times k}$.

\vpara{Step 2: Unified propagation with orthogonal subspaces.}
We first project $\mathbf{X}$ into $\mathbf{H}$:
\begin{align}
\mathbf{H}=\mathbf{X}\mathbf{W},
\end{align}
where $\mathbf{W}\in\mathbb{R}^{k\times k'}$ is learnable.
We then define two subspace bases $\mathbf{W}_s\in\mathbb{R}^{k'\times k_s}$ and $\mathbf{W}_d\in\mathbb{R}^{k'\times k_d}$ and enforce orthogonality:
\begin{align}
\mathbf{W}_s^\top\mathbf{W}_d=\mathbf{0}.
\end{align}
We implement this constraint by adding the orthogonality loss
\begin{align}
\label{formula:lo}
l_o=\|\mathbf{W}_s^\top\mathbf{W}_d\|_F^2,
\end{align}
while noting that numerical orthogonalization (e.g., Gram--Schmidt) is also applicable~\cite{bjorck1994numerics,leon2013gram,bjorck1967solving}.

For static links $\mathbf{A}_s$, we perform hierarchical aggregation where messages are mainly transformed in the static subspace while retaining complementary information from the dynamic subspace:
\begin{align}
h_j^* &= \Big(\alpha\, h_j\mathbf{W}_s+\!\!\sum_{\mathbf{A}_s(i,j)=1}\!\!\beta_i\, h_i^*\mathbf{W}_s\Big)\mathbf{W}_s^+\nonumber\\
&+\frac{1}{n+1}\Big(h_j\mathbf{W}_d+\!\!\sum_{\mathbf{A}_s(i,j)=1}\!\!h_i^*\mathbf{W}_d\Big)\mathbf{W}_d^+,
\label{formula:sg}
\end{align}
where $h$ is the current representation, $h^*$ is the updated representation (with $h^*=h$ for leaf nodes), $n$ is the number of children of $v_j$, and $\alpha,\{\beta_i\}$ are learnable aggregation weights (analogous to attention weights in GAT~\cite{velivckovic2017graph}).
$\mathbf{W}^+$ denotes the pseudo-inverse of $\mathbf{W}$.

For dynamic links $\mathbf{A}_d$, we use a temporal encoder (GRU by default) to propagate along time:
\begin{align}
h_j^* &= \underset{\mathbf{A}_d(i,j)=1}{T_{\text{encoder}}}\big(\text{state}(i),\, h_j\mathbf{W}_d\big)\mathbf{W}_d^+\nonumber
\\&+\frac{1}{2}(h_j+h_i^*)\mathbf{W}_s\mathbf{W}_s^+,
\label{formula:di}
\end{align}
where $\text{state}(i)$ is the hidden state carried from the previous record.
Eqs.~\ref{formula:sg}--\ref{formula:di} implement non-interfering static/dynamic processing: static aggregation primarily operates in $\mathbf{W}_s$ and dynamic reasoning in $\mathbf{W}_d$, while both preserve complementary information to maintain a shared semantic space.
We apply these updates bottom-up from leaves to the root and obtain $h_r^*$ as the basic entity representation:
\begin{align}
h_{\text{basic}}=h_r^*.
\end{align}

\vpara{Step 3: Lightweight hyper-structure via SSAgg.}
Explicitly selecting $\mathcal{V}'$ and constructing $\mathcal{E}'$ can be expensive.
We therefore propose \textbf{S}elective \textbf{S}emantic \textbf{Agg}regation (SSAgg), which implicitly simulates task-specific hyper-structures by re-propagating messages over a unified adjacency matrix:
\begin{align}
\mathbf{A}'=\mathbf{A}_s+\mathbf{A}_d\in\mathbb{R}^{N\times N}.
\end{align}
SSAgg computes attention scores and selectively aggregates incoming messages:
\begin{align}
m_{i,j} &= \psi(h_i^*, h_j^*),\quad \text{for } \mathbf{A}'(i,j)=1,
\label{formula:ssagg_1}\\
\beta_{i,j}&=\frac{\exp(m_{i,j})}{\exp(\lambda)+\sum_{\mathbf{A}'(k,j)=1}\exp(m_{k,j})},
\label{formula:ssagg_2}\\
h_j'&=\sum_{\mathbf{A}'(i,j)=1}\beta_{i,j}\, h_i^*\mathbf{W}',
\label{formula:ssagg_3}
\end{align}
where $\psi$ is a learnable scoring function, $\mathbf{W}'$ is learnable, and $\lambda$ controls selective suppression.
Compared with standard softmax attention, $\exp(\lambda)$ in the denominator provides a tunable ``gate'' that can down-weight all incoming messages, enabling SSAgg to mimic node/edge selection (formalized in Section~\ref{sec:tg}).
We denote the resulting root representation as
\begin{align}
h_{\text{enhanced}}=h_r'.
\end{align}

\vpara{Training objective.}
We concatenate the basic and enhanced representations for prediction:
\begin{align}
h_{\text{final}}=[h_{\text{basic}}\|h_{\text{enhanced}}],
\end{align}
and optimize the downstream loss with orthogonality regularization:
\begin{align}
\label{formula:o_loss}
l = l_{\text{downstream}} + \gamma\, l_o,
\end{align}
where $\gamma$ balances the orthogonality constraint.

\subsection{Theoretical Guarantee}
\label{sec:tg}

In this section, we provide theoretical justification for the effectiveness of SSAgg.
Recall that explicit hyper-structure construction typically involves two expensive steps: selecting task-relevant nodes and generating new edges among them.
SSAgg circumvents both steps by selectively suppressing or propagating messages over the original graph.
We formalize this intuition through the following propositions.
\begin{proposition}
\label{propo:p1}
Given a central node $v_c$, a set of neighbor nodes $\{v_i\}_{i=1}^n$ with directed edges pointing to $v_c$, and any fixed value of $\lambda$, there exists a scoring function $\psi$ of the form defined in Formula~\ref{formula:ssagg_1} such that
\begin{align}
\forall v_k \in \{v_i\}_{i=1}^n,\ \forall \delta > 0,\ \beta_{k,c} < \delta,
\end{align}
where $\beta_{k,c}$ denotes the aggregation coefficient of node $v_k$ computed according to Formula~\ref{formula:ssagg_2}.
\end{proposition}

The proof is provided in Appendix~\ref{appendix:p1}.
Proposition~\ref{propo:p1} shows that, by introducing the hyperparameter $\lambda$, SSAgg can arbitrarily suppress the contribution of any selected neighbor node.
In other words, SSAgg can effectively block message propagation from specific nodes.
This selective suppression is functionally equivalent to node selection in explicit hyper-structure construction.

We next extend this result to show that SSAgg is expressive enough to match the root-level representations produced by a broad class of hyper-structure GNNs.
Specifically, we demonstrate that SSAgg can reproduce the representation obtained by applying any graph neural network to any given hyper-structure graph.

\begin{proposition}
\label{propo:p2}
Given the original graph $G = (\mathcal{V}, \mathcal{E})$ and an arbitrary hyper-structure graph
$G^\prime = (\mathcal{V}^\prime, \mathcal{E}^\prime)$ with $\mathcal{V}^\prime \subseteq \mathcal{V}$,
for any graph neural network $f$ used to encode $G^\prime$, there exists a set of SSAgg parameters $\theta$ such that
\begin{align}
\underset{\theta}{\mathrm{SSAgg}}(G) = h_r^\prime = f(\mathcal{V}^\prime, \mathcal{E}^\prime),
\end{align}
where $\mathrm{SSAgg}(\cdot)$ denotes the process of generating the enhanced root representation $h_r^\prime$ in Step~3.
\end{proposition}

The detailed proof is provided in Appendix~\ref{appendix:p2}.
Together, Propositions~\ref{propo:p1} and~\ref{propo:p2} establish that SSAgg can implicitly realize both node selection and structure construction.
Consequently, SSAgg is theoretically guaranteed to match the representational power of explicitly constructed hyper-structures, while operating directly on the original graph with substantially lower computational overhead.

%% file: 05_experiment.tex
\section{Experiment}
In this section, we systematically evaluate the effectiveness of UniSAGE across multiple tasks.
Experiments are conducted on both public benchmarks and a real-world industrial dataset.
We also perform comprehensive ablation studies to analyze the contribution of each component in UniSAGE.

\subsection{Overall Setup}
All raw data are stored in JSON format.
We compare UniSAGE with four categories of baseline methods:
\begin{itemize} 
\item Traditional Feature Extraction Models: we include traditional models \textit{MLP}~\cite{hornik1991approximation} and \textit{LightGBM}~\cite{ke2017lightgbm} as baselines. These models directly filter and extract features from the input to meet the requirements of downstream tasks. 
\item Static Graph Neural Networks: we also compare UniSAGE with widely-used static GNN models including \textit{GCN}~\cite{kipf2016semi}, \textit{GraphSAGE}~\cite{hamilton2017inductive} and \textit{GAT}~\cite{velivckovic2017graph}. These methods operate on a fixed graph structure. 
\item Dynamic Graph Neural Networks: to evaluate performance in evolving scenarios, we include dynamic GNNs such as \textit{TGN}~\cite{rossi2020temporal}, \textit{TGAT}~\cite{xu2020inductive}, \textit{DyGFormer}~\cite{yu2023towards} and \textit{FreeDyG}~\cite{tian2024freedyg}. These models capture temporal dependencies and structural changes over time. 
\item Relational Deep Learning: we also consider \textit{RDL}~\cite{fey2023relational} and \textit{RelGNN}~\cite{chen2025relgnn} as baselines, which are designed to model complex structured relationships in data. 
\end{itemize}
Detailed descriptions of all baselines are provided in Appendix~\ref{sec:appendix_baselines}.
For traditional feature extraction models, we use a language model to encode JSON text into representations, followed by task-specific predictors.
For static GNN baselines, we construct a graph following Section~\ref{sec:pi} (Step~1), obtain node features via a language model, and apply GNNs on the combined adjacency matrix $\mathbf{A}=\mathbf{A}_s+\mathbf{A}_d$.
Dynamic GNNs follow a similar construction, treating dynamic attributes as time-evolving nodes.
For RDL-based models, we strictly follow their original graph construction and learning procedures.
For UniSAGE, we adopt GRU~\cite{chung2014empirical} as the temporal encoder in Formula~\ref{formula:di}.
Hyperparameters and implementation details are reported in Appendix~\ref{appendix:hys}.

\subsection{Results on RelBench}

\vpara{Dataset Introduction.}
RelBench~\cite{robinson2024relbench} is a public benchmark for learning on relational databases, covering domains such as e-commerce, social networks, and healthcare, with dataset sizes ranging from 74K to 41M entities.
We convert all tabular data into JSON format; details are provided in Appendix~\ref{sec:appendix_datasets}.
Due to the large scale of certain datasets, some baselines cannot be evaluated on full data.
Following prior practice, we uniformly resample 2,000/500/500 instances for training/validation/testing across all tasks.
Full-dataset results of UniSAGE are reported in Appendix~\ref{appendix:complete_results}.

\begin{table*}[tbp]
\centering
    \caption{Entity classification results (AUC-ROC \% $\uparrow$, higher is better) on RelBench datasets. The optimal results are \textbf{bolded}, and the sub-optimal results are \underline{underlined}.
    }
    \label{tab:entity_classification}
    \resizebox{0.9\textwidth}{!}{
    \begin{tabular}{c|cc|cc|cc|cc|c|cc|c}
    \toprule
    Dataset & \multicolumn{2}{c|}{rel-amazon} & \multicolumn{2}{c|}{rel-avito} & \multicolumn{2}{c|}{rel-event} & \multicolumn{2}{c|}{rel-f1} & rel-hm & \multicolumn{2}{c|}{rel-stack} & rel-trial\\
    \midrule
    Task & \makecell{user-\\churn} & \makecell{item-\\churn} & \makecell{user-\\visits} & \makecell{user-\\clicks} & \makecell{user-\\repeat} & \makecell{user-\\ignore} & \makecell{driver-\\dnf} & \makecell{driver-\\top3} & \makecell{user-\\churn} & \makecell{user-\\engage} & \makecell{user-\\badge} & \makecell{study-\\outcome} \\
    \midrule
    MLP& 50.24 & 50.79 & 54.92 & 50.64	& 66.26	& 72.83	& 66.30	& 64.53	& 50.05	& 57.61	& 66.72	& 52.78  \\
    
    LightGBM & 50.36	& 50.84	& 57.74	& 51.34	& 67.05	& 77.58	& 70.35	& 70.09	& 50.66	& 66.69	& 67.01	& 53.87\\

    \midrule
    
    GCN	& 51.17	& 70.71	& 60.59	& 57.53	& 67.96	& 78.28	& 74.51	& 74.76	& 51.83	& 73.48	& 73.34	& 55.45 \\

    GraphSAGE & 51.96	& 70.81	& 61.17	& 56.39	& 67.09	& 78.26	& 71.82	& 74.99	& 53.57	& 70.32	& 74.92	& 54.73 \\
    
    GAT	& 52.14	& 73.42	& 61.54	& 62.41	& 67.42	& 80.77	& 75.39	& \underline{76.11}	& 57.56	& 79.31	& 79.61	& \underline{56.12}\\

    \midrule
    
    TGN	& 51.68	& 65.10	& 58.29	& 52.65	& 67.45	& 80.48	& 70.70	& 70.25	& \underline{69.71}	& 71.80	& 80.45	& 55.46\\
    
    TGAT & 56.19	& 56.70	& 62.51	& 52.95	& 67.30	& 79.32	& 73.24	& 73.80	& 65.65	& 76.32	& \underline{88.34}	& 54.16\\
    
    DyGFormer & 52.11	& 57.12	& 58.69	& 59.77	& 67.66	& 77.70	& \underline{75.71}	& 72.37	& 65.15	& 80.08	& 86.81	& 56.00\\
    
    FreeDyG & 51.54 & 54.52	& 65.34	& 65.75	& 68.87	& 80.52	& 71.51	& 70.61	& 65.22	& 80.13	& 83.20	& 54.48\\

    \midrule
    
    RDL &	\underline{58.06}	& \underline{81.43}	& \underline{68.27}	& \underline{71.09}	& \underline{69.91}	& 80.37	& 70.35	& 73.03	& 68.63	& \underline{86.52}	& 84.16	& 53.34\\

    RelGNN & 55.70 & 80.80 & 65.80 & 68.57 & 68.01 & \underline{81.37} & 70.90 & 73.60 & 66.26 & 82.19 & 84.50 & 51.27\\

    \midrule
    
    UniSAGE & \textbf{59.24}	& \textbf{81.69}	& \textbf{70.60}	& \textbf{71.88}	& \textbf{70.38}	& \textbf{87.05}	& \textbf{82.57}	& \textbf{82.20}	& \textbf{71.40}	& \textbf{88.92}	& \textbf{90.21}	& \textbf{57.17}\\


    
    \bottomrule
    \end{tabular}
    }
\end{table*}

\vpara{Entity Classification Results.}
Table~\ref{tab:entity_classification} reports the AUC-ROC results on all entity classification tasks.
We summarize the key observations as follows.

\vpait{1. UniSAGE consistently outperforms all baselines.}
UniSAGE achieves the best results on all evaluated tasks.
On average, it improves over traditional feature-based models by 15.93\% and over GNN-based methods by 9.22\%.
In several tasks, UniSAGE exceeds the second-best method by nearly 10\%, demonstrating strong and stable performance across diverse datasets.

\vpait{2. UniSAGE robustly integrates structural and temporal information.}
While static GNNs outperform traditional baselines by 6.71\% on average, their performance varies significantly across tasks.
Dynamic GNNs introduce temporal signals but yield only marginal gains over static counterparts (0.52\% on average), and even underperform in some cases.
In contrast, UniSAGE outperforms dynamic GNNs by 8.95\% on average, indicating its superior ability to jointly model heterogeneous structures and asynchronous temporal behaviors.

\vpait{3. UniSAGE generalizes beyond schema-specific relational models.}
Compared with RDL and RelGNN, which are tailored to RelBench schemas, UniSAGE achieves an average improvement of nearly 5\%.
For certain tasks, such as rel-f1 driver-dnf, the gain exceeds 10\%, highlighting UniSAGE’s strong generalization capability on complex relational data.

\vpara{Entity Regression Results.}
Results on entity regression tasks are reported in Appendix~\ref{appendix:entity_regression}.
UniSAGE consistently achieves the best performance across all datasets, outperforming traditional models, static GNNs, dynamic GNNs, and RDL methods by 22.34\%, 14.05\%, 10.87\%, and 8.28\% on average, respectively.
These results further confirm the robustness of UniSAGE across different learning objectives.

\subsection{Results on Real-World Data}

\vpara{Dataset Introduction.}
We use a large-scale real-world financial behavior dataset, referred to as \textbf{UserBehavior}, which contains 191,927 user records with both demographic attributes (\emph{e.g.}, gender, education, employment status) and temporal behavioral logs (\emph{e.g.}, loan applications, credit account transactions).
The dataset is stored in JSON format, where each user profile consists of a set of static features along with a sequence of dynamic events.
The downstream task is formulated as a binary anomaly-detection problem, aiming to predict whether a user poses a potential credit risk.
We use 115,156 records (60\%) for training, 30,097 records (15\%) for validation, and 46,674 records (25\%) for testing.
All baselines follow the same preprocessing protocol as in RelBench, and we use the same text encoder~\cite{reimers-2019-sentence-bert} for feature construction.
To ensure data privacy and regulatory compliance, the dataset is rigorously anonymized and all sensitive information is removed.
Although the dataset is not publicly available at present, a public release is under consideration.
Detailed dataset statistics are provided in Appendix~\ref{appendix:ub}.

\begin{table}[h]
    \centering
    \caption{The experimental results (AUC-ROC \% $\uparrow$, higher is better) on the real-world industrial dataset UserBehavior.}
    \label{tab:user_behavior}
    \resizebox{\columnwidth}{!}{
    \begin{tabular}{c|ccccc}
    \toprule
    Model & MLP & LightGBM & GCN & GraphSAGE & GAT\\
    \midrule
    Result & 50.10 & 52.12 & 52.64 & 52.34 & 53.80 \\
    \midrule
    \midrule
    Model & TGN & DyGFormer & TGAT & FreeDyG & UniSAGE\\
    \midrule
    Result & 53.91 & 54.61 & 54.03 & 54.12 & \textbf{59.06}\\
    \bottomrule
    \end{tabular}
    }
\end{table}

\vpara{Experimental Results.}
Table~\ref{tab:user_behavior} reports the performance of all baselines on UserBehavior.
The results indicate that this task is challenging: simple text encoding followed by an MLP fails to produce a competitive model.
Incorporating graph structure yields measurable gains, but the improvements remain modest, suggesting that existing static GNNs struggle to fully exploit the underlying hierarchical organization.
Dynamic GNNs that model temporal semantics provide additional benefits, yet the overall performance is still limited, indicating that effectively leveraging asynchronous behavioral logs remains difficult.
In contrast, UniSAGE achieves a substantial performance boost.
Compared to the best competing baseline (DyGFormer), UniSAGE delivers a relative improvement of nearly 10\%.
These results underscore the advantage of UniSAGE’s unified modeling of static and dynamic attributes, demonstrating its suitability for complex real-world industrial data and its ability to extract task-relevant signals in challenging scenarios.

\subsection{Ablation Study}
We conduct ablation studies to quantify the contribution of each component in UniSAGE.
Specifically, we evaluate on the rel-f1 driver-dnf task and the UserBehavior dataset, and compare the following variants:

\vpait{1. UniSAGE without unified representation and without SSAgg} (UniSAGE w/o UR and w/o SSAgg):
This variant uses only Step~1 in Section~\ref{sec:pi} to construct the graph and perform bottom-up propagation from leaf nodes to the root.
It removes unified representation learning in Step~2 and does not apply SSAgg for representation enhancement.

\vpait{2. UniSAGE without orthogonality loss and without SSAgg} (UniSAGE w/o OL and w/o SSAgg):
Based on Variant~1, this model includes the parameter subspaces $\mathbf{W}_s$ and $\mathbf{W}_d$ in Step~2 but removes the orthogonality loss in Formula~\ref{formula:lo}, i.e., no constraint is imposed to enforce orthogonality between the two spaces.
SSAgg is also excluded.

\vpait{3. UniSAGE without SSAgg} (UniSAGE w/o SSAgg):
This variant includes both Step~1 and Step~2 but removes SSAgg, i.e., it does not perform representation enhancement via hyper-structure reasoning.

\vpait{4. UniSAGE}:
The full model that includes all components described in our framework.

\begin{figure}[h]
\centering
\includegraphics[width=\columnwidth]{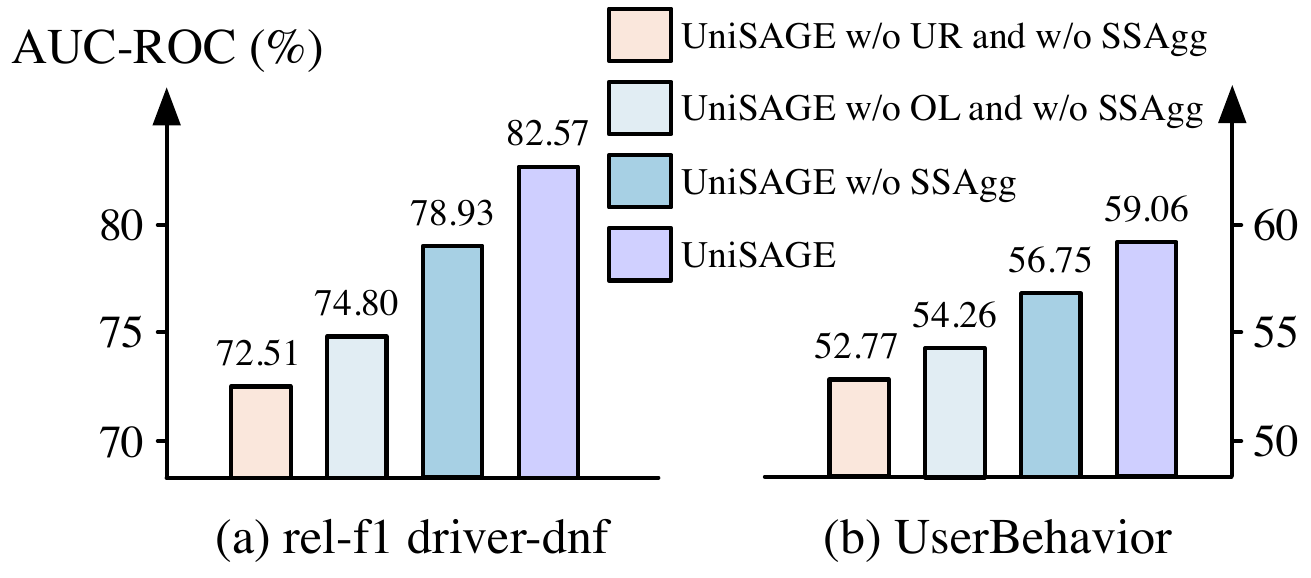}
\caption{The ablation results for UniSAGE.}
\label{fig:ablation_study}
\end{figure}

Figure~\ref{fig:ablation_study} presents the ablation results.
Using only the graph structure yields clearly suboptimal performance (UniSAGE w/o UR and w/o SSAgg), indicating that naive propagation is insufficient.
Introducing separate static aggregation and dynamic inference improves performance (UniSAGE w/o OL and w/o SSAgg), highlighting the importance of modeling both attribute types.
However, without the orthogonality loss, the static and dynamic representation spaces may drift during training, which limits the effectiveness of this unified modeling.
Adding the orthogonality constraint stabilizes the shared semantic space and further improves performance (UniSAGE w/o SSAgg).
Finally, incorporating SSAgg enables task-relevant hyper-structure reasoning and strengthens cross-type interactions, leading to the best overall results (UniSAGE).
Overall, these results confirm the necessity of each component and align well with our theoretical motivation and empirical design choices.

\subsection{Additional Results and Analysis}
\label{sec:appendix_navigation}

To improve readability and focus in the main paper, we defer several complementary experimental results and analyses to the appendix.
Specifically, Appendix~\ref{sec:appendix_task_stats} provides detailed statistics of all datasets and tasks used in RelBench.
Complete results on the full (non-resampled) RelBench datasets are reported in Appendix~\ref{appendix:complete_results}.
Appendix~\ref{appendix:time} presents a detailed empirical analysis of the runtime efficiency of UniSAGE, including comparisons across different datasets and task settings.
Appendix~\ref{appendix:hyperparameter} provides an in-depth hyperparameter sensitivity study, analyzing the impact of key hyperparameters and demonstrating the robustness of UniSAGE under reasonable configurations.

%% file: 06_conclusion.tex
\section{Conclusion}

In this paper, we propose UniSAGE, a unified framework for modeling complex data with both static and dynamic attributes.
UniSAGE addresses key limitations of existing approaches by constructing a global attribute graph, learning unified representations through orthogonal parameter subspaces, and leveraging hyper-structures to enhance downstream tasks.
Extensive experiments on public benchmarks and a real-world industrial dataset demonstrate that UniSAGE consistently outperforms strong baselines, achieving improvements of up to 10\% on challenging tasks.
By combining automation, strong generalizability, and effective modeling of implicit static–dynamic correlations, UniSAGE provides a practical and robust solution for real-world data-intensive applications.

\section{Acknowledgments}
This work was supported by National Natural Science Foundation of China (No. 62322606, No. 62441605), and the Fundamental Research Funds for the Central Universities.

%% file: 07_appendix.tex
\appendix

\section{Theoretical Derivation}

\subsection{Proof of Proposition \ref{propo:p1}}
\label{appendix:p1}

\textbf{Proposition 4.1.}\textit{
Given a central node $v_c$, a set of neighbor nodes $\{v_i\}_{i=1}^n$ with directed edges pointing to $v_c$, and any fixed value of $\lambda$, there exists a scoring function $\psi$ of the form defined in Formula~\ref{formula:ssagg_1} such that
\begin{align}
\forall v_k \in \{v_i\}_{i=1}^n,\ \forall \delta > 0,\ \beta_{k,c} < \delta,
\end{align}
where $\beta_{k,c}$ denotes the aggregation coefficient of node $v_k$ computed according to Formula~\ref{formula:ssagg_2}.
}

\begin{proof}
Let $M = \max_{i=1,\dots,n} m_{i,c}$ be the largest attention score among all neighbors. 
Then for any node $v_k \in \{v_i\}_{i=1}^n$, we have:
\begin{align}
\beta_{k,c} = \frac{\exp(m_{k,c})}{\exp(\lambda) + \sum_{i=1}^{n} \exp(m_{i,c})}
\le \frac{\exp(M)}{\exp(\lambda) + n \cdot \exp(M)}
\end{align}
Then we define:
\begin{align}
\beta_{\max} = \frac{\exp(M)}{\exp(\lambda) + n \cdot \exp(M)} = \frac{1}{\exp(\lambda - M) + n}
\end{align}
To guarantee $\beta_{k,c} < \delta$, it suffices that:
\begin{align}
\frac{1}{\exp(\lambda - M) + n} < \delta
\end{align}
Rearranging gives:
\begin{align}
\exp(\lambda - M) > \frac{1}{\delta} - n.
\end{align}
Since $\delta > 0$, the right-hand side is finite as long as $\delta < \frac{1}{n}$.
Thus, we can always choose $M$ sufficiently smaller than $\lambda$ such that:
\begin{align}
M < \lambda - \log\left(\frac{1}{\delta} - n\right)
\end{align}
This implies that if we define $\psi(\cdot)$ such that all $m_{i,c} \le M$, then $\beta_{i,c} < \delta$ for all $i$.
Such a function $\psi$ trivially exists (\emph{e.g.}, a bounded scoring function with controlled output range).
At this point, we have proven that for any given $\delta > 0$, we can always choose a scoring function $\psi(\cdot)$ with bounded maximum output to ensure $\beta_{k,c} < \delta$ for all $k$.
\end{proof}

\subsection{Proof of Proposition \ref{propo:p2}}
\label{appendix:p2}

\textbf{Proposition 4.2.}\textit{
Given the original graph $G = (\mathcal{V}, \mathcal{E})$ and an arbitrary hyper-structure graph
$G^\prime = (\mathcal{V}^\prime, \mathcal{E}^\prime)$ with $\mathcal{V}^\prime \subseteq \mathcal{V}$,
for any graph neural network $f$ used to encode $G^\prime$, there exists a set of SSAgg parameters $\theta$ such that
\begin{align}
\underset{\theta}{\mathrm{SSAgg}}(G) = h_r^\prime = f(\mathcal{V}^\prime, \mathcal{E}^\prime),
\end{align}
where $\mathrm{SSAgg}(\cdot)$ denotes the process of generating the enhanced root representation $h_r^\prime$ in Step~3.
}

\begin{proof}
For analytical simplicity, we initially assume the hyper-structure graph $G^{\prime} = (\mathcal{V}^{\prime}, \mathcal{E}^{\prime})$ is encoded using a single-layer GAT model followed by a mean pooling operation over all its nodes. 
The GAT representation for each node $j \in \mathcal{V}^{\prime}$ can be expressed as:
\begin{align}
h_j^{\text{HS}} =  \sum_{k \in \mathcal{N}(j)} \alpha_{j,k} h_k^* \cdot  \mathbf{W}_{\text{HS}}
\end{align}
where $h_k^*$ is the initial representation of node $v_k$, $\alpha_{j,k}$ is the normalized attention coefficient for neighbor $v_k$ of node $v_j$, and $\mathbf{W}_{\text{HS}}$ is a learnable linear transformation matrix.
Then, the final graph representation is computed as:
\begin{align}
h_{G^\prime}^{\text{HS}} = \frac{1}{|\mathcal{V}^{\prime}|} \sum_{j \in \mathcal{V}^{\prime}} h_j^{\text{HS}}.
\end{align}
Each $h_j^{\text{HS}}$ is a weighted sum over $\{ h_k^* \}_{k \in \mathcal{V}}$, so $h_{G^\prime}^{\text{HS}}$ itself is also a weighted sum of $\{ h_k^* \}_{k \in \mathcal{V}}$, which can be expressed as:
\begin{align}
h_{G^\prime}^{\text{HS}} = \sum_{k \in \mathcal{V}} \lambda_k h_k^*,
\end{align}
where the weights $\{ \lambda_k \}$ are depend on the hyper-structure and parameters of the GAT and pooling operation.
Now, recall that in the SSAgg mechanism, the information of the entire graph is aggregated layer by layer to the root node, and the root representation can be obtained by:
\begin{align}
h_r^\prime = \sum_{k \in \mathcal{V}} \gamma_k h_k^*\cdot\mathbf{W}^\prime,
\end{align}
where $\{ \gamma_k \}$ are aggregation weights learned through bottom-up semantic aggregation via selective paths.
Since both $h_r^\prime$ and $h_{G^\prime}^{\text{HS}}$ are linear combinations of $\{ h_k^* \}$, there always exists a set of SSAgg parameters $\theta$ such that:
\begin{align}
\gamma_k = \lambda_k \quad \text{for all } k \in \mathcal{V}
\end{align}
Therefore, by setting the SSAgg weights $\gamma_k$ to match $\lambda_k$ and making $\mathbf{W}_{\text{HS}}=\mathbf{W}^\prime$, we ensure that:
\begin{align}
h_r^\prime = \sum_{k \in \mathcal{V}} \gamma_k h_k^*\cdot\mathbf{W}^\prime = \sum_{k \in \mathcal{V}} \lambda_k h_k^*\cdot\mathbf{W}_{\text{HS}} = h_{G^\prime}^{\text{HS}}
\end{align}
We have found a set of parameters $\theta$ for SSAgg that satisfies:
\begin{align}
\underset{\theta}{\text{\rm SSAgg}}(G) = f(G^\prime)
\end{align}
Thus, Proposition \ref{propo:p2} is proved.
\end{proof}

\section{Experiment Details}
\label{sec:appendix_experiment}

\subsection{Hyper-Parameter Settings and Obtained Results}
\label{appendix:hys}

For all models, we apply the dropout technique with dropout rates selected from $[0, 0.1, 0.3, 0.5, 0.8]$.
Additionally, we utilize the Adam optimizer, choosing learning rates from $[10^{-3}, 5 \times 10^{-3}, 10^{-2}]$ and weight decay values from $[0, 5 \times 10^{-4}]$.
Regarding model architecture, UniSAGE consistently maintains a one-time, layer-wise aggregation from leaf nodes to the root node, while the number of GNN/MLP layers is selected from $[1, 2, 3, 4, 5]$, the hidden dimension from $[8, 16, 32, 64, 128, 256]$, and the number of heads for attentive models from $[1, 2, 4, 8]$.
For the SSAgg component in UniSAGE, the hyperparameter $\lambda$ is selected from $[0, 1, 2]$, and the hyperparameter $\gamma$ in Formula \ref{formula:o_loss} is selected from $[0.01, 0.1, 0.5]$.
For the unique hyperparameters of other baselines, we follow the hyperparameter search space provided in their original papers for tuning.
We conduct five rounds of experiments with different random seeds for each setting and report the average results.

\subsection{Baseline Model Descriptions}
\label{sec:appendix_baselines}
This section offers a concise overview of the baseline models used in our experiments, categorized according to their main approaches to feature processing and graph representation.

\vpait{Traditional Feature Extraction Models}

\vpara{MLP}~\cite{hornik1991approximation}.
MLP is a basic feed-forward neural network composed of multiple fully connected layers with nonlinear activations. It excels at learning complex combinations of input features and is suitable for tabular data or flattened feature vectors, but it has limited ability to capture underlying structural or temporal relationships.

\vpara{LightGBM}~\cite{ke2017lightgbm}.
LightGBM is an efficient implementation of Gradient Boosting Decision Trees (GBDT), known for fast training and low memory consumption. It uses histogram-based algorithms and a leaf-wise tree growth strategy, making it highly effective for large-scale tabular datasets and high-dimensional features.

\vpait{Static Graph Neural Networks}

\vpara{GCN}~\cite{kipf2016semi}.
GCN is a foundational graph neural network that updates node representations by aggregating features from their immediate neighbors. It provides an efficient approach for semi-supervised learning on graphs but is mainly designed for static graph structures without temporal changes.

\vpara{GraphSAGE}~\cite{hamilton2017inductive}.
GraphSAGE is an inductive GNN framework that learns aggregator functions to generate embeddings for unseen nodes or new graphs. Instead of learning separate embeddings per node, it learns how to sample and aggregate neighborhood features, enabling strong generalization to unseen data.

\vpara{GAT}~\cite{velivckovic2017graph}.
GAT integrates masked self-attention mechanisms to assign different importance weights to neighbors during feature aggregation. This allows the model to focus on more relevant nodes for a task without requiring full graph knowledge, achieving state-of-the-art results on many benchmarks.

\vpait{Dynamic Graph Neural Networks}

\vpara{TGN}~\cite{rossi2020temporal}.
TGN is a general framework for deep learning on dynamic graphs that introduces a memory module to store each node’s evolving state over time. When new interactions happen, the model updates the memory of the involved nodes, enabling effective capture of temporal patterns and generation of time-aware node embeddings.

\vpara{TGAT}~\cite{xu2020inductive}.
TGAT is an inductive model for dynamic graphs that enhances the graph attention mechanism with temporal encoding. It encodes the timestamps of node interactions and integrates them into the self-attention computation, allowing it to effectively capture temporal dependencies in the data.

\vpara{DyGFormer}~\cite{yu2023towards}.
DyGFormer leverages the Transformer architecture to model dynamic graphs by treating temporal interactions as sequences. This approach enables it to capture long-range temporal dependencies alongside higher-order structural information effectively.

\vpara{FreeDyG}~\cite{tian2024freedyg}.
FreeDyG is a novel approach that analyzes dynamic graph signals in the frequency domain. By learning temporal evolution patterns from this perspective, the model effectively captures underlying periodicities and trends while filtering out noise, resulting in improved predictive performance.

\vpait{Relational Deep Learning}
RelGNN is a relational graph neural network tailored for relational databases, which explicitly exploits primary foreign key structures. By introducing atomic routes and composite message passing, RelGNN enables direct single-hop information exchange across many-to-many relations, reducing redundant aggregation and highlighting key predictive signals, leading to more efficient and accurate relational deep learning.

\vpara{RDL}~\cite{fey2023relational}.
RDL is the official baseline framework introduced with the RelBench benchmark, specifically designed for representation learning on relational databases. It offers a standardized approach to automatically convert relational tables into heterogeneous graphs and employs a custom relational graph neural network (RGNN) to learn entity representations for predictive tasks within the database.

\vpara{RelGNN}~\cite{chen2025relgnn}.
RelGNN is a relational graph neural network tailored for relational databases, which explicitly exploits primary--foreign key structures. By introducing atomic routes and composite message passing, RelGNN enables direct single-hop information exchange across many-to-many relations, reducing redundant aggregation and highlighting key predictive signals, leading to more efficient and accurate relational deep learning.

\begin{table*}[tbp]
\centering
\caption{Statistics of the predictive tasks used for evaluation, with data extracted from the original RelBench dataset splits (as referenced in Table 1 of the main text). \#Unique Entities refers to the number of unique root entities in the task.}
\label{tab:task_statistics_final}
\begin{tabular}{lllrrrrc}
\toprule
\multirow{2}{*}{\textbf{Dataset}} & \multirow{2}{*}{\textbf{Task name}} & \multirow{2}{*}{\textbf{Task type}} & \multicolumn{3}{c}{\textbf{\#Rows of training table}} & \textbf{\#Unique} & \textbf{\%train/test} \\
\cmidrule(lr){4-6}
& & & \textbf{Train} & \textbf{Validation} & \textbf{Test} & \textbf{Entities} & \textbf{Entity Overlap} \\
\midrule
\multirow{4}{*}{rel-amazon} 
& user-churn & entity-cls & 4,732,555 & 409,792 & 351,885 & 1,585,983 & 88.0 \\
& item-churn & entity-cls & 2,559,264 & 177,689 & 166,842 & 416,352 & 93.1 \\
& user-ltv & entity-reg & 4,732,555 & 409,792 & 351,885 & 1,585,983 & 88.0 \\
& item-ltv & entity-reg & 2,707,679 & 166,978 & 178,334 & 427,537 & 93.5 \\
\midrule
\multirow{3}{*}{rel-avito} 
& ad-ctr & entity-reg & 5,100 & 1,766 & 1,816 & 4,997 & 59.8 \\
& user-clicks & entity-cls & 59,454 & 21,183 & 47,996 & 66,449 & 45.3 \\
& user-visits & entity-cls & 86,619 & 29,979 & 36,129 & 63,405 & 64.6 \\
\midrule
\multirow{3}{*}{rel-event} 
& user-attendance & entity-reg & 19,261 & 2,014 & 2,006 & 9,694 & 14.6 \\
& user-repeat & entity-cls & 3,842 & 268 & 246 & 1,514 & 11.5 \\
& user-ignore & entity-cls & 19,239 & 4,185 & 4,010 & 9,799 & 21.1 \\
\midrule
\multirow{3}{*}{rel-f1} 
& driver-dnf & entity-cls & 11,411 & 566 & 702 & 821 & 50.0 \\
& driver-top3 & entity-cls & 1,353 & 588 & 726 & 134 & 50.0 \\
& driver-position & entity-reg & 7,453 & 499 & 760 & 826 & 44.6 \\
\midrule
\multirow{2}{*}{rel-hm} 
& user-churn & entity-cls & 3,871,410 & 76,556 & 74,575 & 1,002,984 & 89.7 \\
& item-sales & entity-reg & 5,488,184 & 105,542 & 105,542 & 1,005,542 & 100.0 \\
\midrule
\multirow{3}{*}{rel-stack} 
& user-engagement & entity-cls & 1,360,850 & 85,838 & 88,137 & 88,137 & 97.4 \\
& user-badge & entity-cls & 3,386,276 & 247,398 & 255,360 & 255,360 & 96.9 \\
& post-votes & entity-reg & 2,453,921 & 156,216 & 160,903 & 160,903 & 97.1 \\
\midrule
\multirow{2}{*}{rel-trial} 
& study-outcome & entity-cls & 11,994 & 960 & 825 & 13,779 & 0.0 \\
& study-adverse & entity-reg & 43,335 & 3,596 & 3,098 & 50,029 & 0.0 \\
\bottomrule
\end{tabular}
\end{table*}

\subsection{Dataset Description and Graph Construction}
\label{sec:appendix_datasets}
In RelBench~\cite{robinson2024relbench}, data is stored as text across multiple tables, enabling researchers to process the tabular data flexibly. 
As one of the most challenging benchmarks to date, RelBench includes a wide range of static attributes as well as numerous dynamic records, such as purchase histories and review logs. 
It supports tasks like entity classification (\emph{e.g.}, predicting customer churn in rel-hm) and entity regression (\emph{e.g.}, forecasting product revenue in rel-amazon).
For our experiments, we first transformed the tabular data into JSON-formatted text for each entity, tailored to the dataset characteristics and downstream tasks. 
We then employed the same language model~\cite{reimers-2019-sentence-bert} used in the original RelBench paper to encode node features from this textual data. 
As noted earlier, traditional feature extraction models use the encoded JSON text directly, while other models integrate graph structures.

This section provides a detailed description of each RelBench dataset used in our experiments.
For each dataset, we summarize its domain, key tables, and the specific graph construction strategy applied by our framework to generate hierarchical JSON representations for the predictive tasks. 
This transformation is a crucial step, enabling the relational data to be effectively processed by graph neural networks.

\vpait{rel-amazon: E-commerce}

\vpara{Dataset Description.}
The rel-amazon dataset models an e-commerce environment focused on book products from the Amazon platform. It includes interconnected tables for customers, products, and reviews. The product table contains numerous static attributes such as price, brand, and category. In contrast, the review table records dynamic, time-stamped interactions between users and products, including ratings and review texts.

\vpara{Graph Construction Strategy.}
Our framework generates two types of graphs from this dataset, corresponding to the two root entities: customers and products. For customer-centric tasks (such as user-churn and user-ltv), the root node represents a customer and is connected to their associated reviews. Each review is treated as a dynamic entity with its own timestamp (review\_time) and links to the corresponding product node, which contains the product’s static attributes. For product-centric tasks (such as item-churn and item-ltv), the root node represents a product and is connected to all its reviews. Each review then links back to the customer who authored it.

\vpait{rel-stack: Q\&A Platform}

\vpara{Dataset Description.}
This dataset originates from the Stack Exchange network, focusing specifically on the statistics Q\&A site (stats-exchange). It models a community-driven platform that includes users, posts (questions and answers), comments, votes, and badges. The data captures both the content and the complex social and reputational interactions among community members.

\vpara{Graph Construction Strategy.}
For user-centric tasks (user-badge, user-engagement), the root user node links to multiple lists representing their activities, including posts created, comments written, votes cast, and badges earned. Each item in these lists is a detailed, time-stamped entity. For the post-votes task, the root post node connects to its author (user) and to lists of related comments, votes, and postLinks, effectively modeling the lifecycle of the content.

\vpait{rel-trial: Clinical Trials}

\vpara{Dataset Description.}
The rel-trial dataset is a comprehensive compilation of clinical trial information sourced from ClinicalTrials.gov. It encompasses detailed data on study protocols, participating sites (facilities), sponsors, medical conditions under investigation, interventions, and reported outcomes, including adverse events.

\vpara{Graph Construction Strategy.}
For study-centric tasks (study-outcome, study-adverse), the root node is the study itself. It connects to static attributes such as design and eligibility criteria, as well as to lists of related entities including sponsors, facilities, conditions, and interventions. The outcomes list is especially complex, with each outcome treated as a dynamic entity that further contains a nested list of its own statistical analyses. Note that the site-success task was not trained due to the extremely large graph sizes generated for each facility, which resulted in prohibitive memory requirements.

\vpait{rel-f1: Formula 1 Racing}

\vpara{Dataset Description.}
This dataset comprises historical data from Formula 1 racing events dating back to 1950. It includes detailed information on drivers, constructors (teams), circuits, and comprehensive race results, such as qualifying positions and final standings.

\vpara{Graph Construction Strategy.}
For all F1 tasks, the graph is driver-centric. The root driver node connects to a list of seasons. Each season node contains a list of races in which the driver participated. Each race node is a detailed entity that includes circuit information and a driver performance sub-object. This sub-object contains the driver’s qualifying result, final race result, and championship standing after that race, chronologically capturing the driver’s entire career progression.

\vpait{rel-hm: E-commerce (Fashion)}

\vpara{Dataset Description.}
The rel-hm dataset contains customer transaction data from the fashion retailer H\&M. It includes detailed customer profiles, article (product) metadata, and a large table of time-stamped transactions that link customers to articles.

\vpara{Graph Construction Strategy.}
For the user-churn task, the root node represents a customer and connects to a list of their purchase history. Each purchase is represented as a transaction node containing details such as date and price, which links to the purchased article along with its complete metadata. For the item-sales task, the root node is an article connected to its purchase history, where each transaction links back to the customer who made the purchase.

\vpait{rel-event: Event Recommendation}

\vpara{Dataset Description.}
This dataset, collected from the Hangtime mobile app, captures user interactions with social events. It includes user profiles, event metadata, social connections between users, and records of event attendance.

\vpara{Graph Construction Strategy.}
For all tasks, the graph is user-centric. The root user node connects to their list of friends, a list of events they created, and a list of event engagements (events they attended but did not create). Each event in these lists is time-stamped and further contains nested lists of attendees and users who expressed interest.

\vpait{rel-avito: Online Advertisements}

\vpara{Dataset Description.}
The rel-avito dataset originates from an online classifieds platform. It includes information about advertisements, user search behavior, ad views (visits), and phone number requests for ads, connecting users to their search queries and subsequent interactions.

\vpara{Graph Construction Strategy.}
For the ad-ctr task, the root ad node connects to lists of its historical interactions, including search appearances, user visits, and phone requests. Each interaction links back to the corresponding user. For user-centric tasks (user-visits, user-clicks), the root user node connects to their search history, where each search nests a list of displayed ads. Additionally, the user node directly connects to lists of their ad visits and phone requests.

\subsection{Task Dataset Statistics}
\label{sec:appendix_task_stats}
To give a clear overview of the scale of our experiments, Table \ref{tab:task_statistics_final} presents detailed statistics for each predictive task used in this study, based on the full RelBench benchmark. The table includes the number of samples in the original dataset as well as in the unsampled training, validation, and test splits, along with the total count of unique entities involved in each task.

\begin{table*}[tbp]
\centering
    \caption{Entity regression results (MAE $\downarrow$, lower is better) on RelBench datasets. The optimal results are \textbf{bolded}, and the sub-optimal results are \underline{underlined}.
    }
    \label{tab:entity_regression}
    \resizebox{0.9\textwidth}{!}{
    \begin{tabular}{c|cc|c|c|c|c|c|c}
    \toprule
    Dataset & \multicolumn{2}{c|}{rel-amazon} & rel-avito & rel-event & rel-f1 & rel-hm & rel-stack & rel-trial\\
    \midrule
    Task & user-ltv & item-ltv & ad-ctr & user-attend & driver-pos & item-sales & post-votes & study-adverse \\
    \midrule
    MLP& 81.832	& 103.575	& 0.063	& 0.358	& 4.455	& 1.093	& 0.579	& 71.512\\
    
    LightGBM & 77.555	& 100.143	& 0.057	& 0.327	& 4.067	& 1.080	& 0.572	& 69.965\\

    \midrule
    
    GCN	& 74.378	& 92.116	& 0.054	& 0.320	& 3.989	& 1.017	& 0.551	& 67.922\\

    GraphSAGE & 75.418	& 93.886	& 0.051	& 0.321	& 3.893	& 1.058	& 0.539	& 67.826\\
    
    GAT	& 76.274	& 91.609	& 0.055	& 0.324	& 3.959	& 1.006	& 0.570	& 66.151\\

    \midrule
    
    TGN	& \underline{67.284}	& 92.288	& 0.079	& 0.311	& 3.702	& 0.987	& 0.546	& 69.071\\
    
    TGAT & 73.483	& 92.233	& 0.082	& 0.314	& 3.720	& 1.063	& 0.534	& 64.295\\
    
    DyGFormer & 73.473	& 89.373	& 0.087	& 0.312	& \underline{3.614}	& 1.055	& 0.544	& 69.068\\
    
    FreeDyG & 69.073	& 90.267	& 0.090	& 0.308	& 3.801	& 1.015	& 0.532	& 64.917\\

    \midrule
    
    RDL & 77.719	& \underline{81.279}	& \underline{0.047}	& 0.299	& 4.192	& 0.923	& \underline{0.501}	& 63.951\\
    
    RelGNN & 80.040 & 83.361 & 0.048 & \underline{0.296} & 4.443 & \underline{0.916} & 0.503 & \underline{63.200}\\

    \midrule
    
    UniSAGE & \textbf{64.439}	& \textbf{81.095}	& \textbf{0.046}	& \textbf{0.293}	& \textbf{3.384}	& \textbf{0.909}	& \textbf{0.499}	& \textbf{60.725}\\


    
    \bottomrule
    \end{tabular}
    }
\end{table*}

\subsection{Entity Regression Results on RelBench}
\label{appendix:entity_regression}

Table~\ref{tab:entity_regression} presents the detailed results of all entity regression tasks on RelBench.
The table includes a comprehensive comparison between UniSAGE and representative baselines, covering traditional feature-based methods, static GNNs, dynamic GNNs, and relational learning approaches.
Across all datasets and regression targets, UniSAGE consistently achieves the lowest prediction error, without relying on task-specific tuning.
The performance improvements are observed uniformly across different domains and entity types, indicating that the advantages of UniSAGE are not confined to a particular dataset or regression objective.
In comparison, traditional feature extraction methods exhibit limited expressive power when modeling complex relational dependencies, resulting in significantly higher regression errors.
Static GNNs benefit from structural information but struggle to capture temporal dynamics, while dynamic GNNs show unstable behavior across tasks, suggesting that naive temporal modeling is insufficient for robust regression performance.
Relational learning methods such as RDL improve over generic GNNs in some cases, but their gains are inconsistent and sensitive to schema-specific assumptions.
Overall, the detailed results in Table~\ref{tab:entity_regression} further corroborate the conclusions drawn in the main paper:
UniSAGE provides a unified and robust solution for entity regression on heterogeneous relational data, effectively integrating structural and dynamic information across diverse real-world scenarios.

\subsection{Results on Complete RelBench Datasets}
\label{appendix:complete_results}

In the main text, we reported comparison results between UniSAGE and various baselines on the sampled RelBench datasets, due to the substantial time and memory demands of some baseline models.
Here, using the entity classification task on rel-f1 as an example, we compare UniSAGE’s performance against selected baselines on the full datasets.
Table \ref{tab:complete_results} presents these experimental results, demonstrating that UniSAGE consistently achieves significant performance advantages on both the sampled and full RelBench datasets.

\begin{table}[h]
    \centering
    \caption{Entity classification results (AUC-ROC \% $\uparrow$, higher is better) on rel-f1.}
    \label{tab:complete_results}
    \begin{tabular}{c|c|c}
    \toprule
    Task & driver-dnf & driver-top3 \\
    \midrule
    MLP &  66.79 &  66.03\\
    LightGBM  & 68.56 &  73.92\\
    \midrule
    RDL  & 72.62 & 75.54\\
    RelGNN & 75.29 & 85.69\\
    \midrule
    UniSAGE  & \textbf{81.17} & \textbf{89.05}\\
    \bottomrule
    \end{tabular}
\end{table}

\subsection{Static of Industrial Dataset UserBehavior}
\label{appendix:ub}

\vpara{Dataset Description.}
The UserBehavior dataset is a real-world industrial dataset comprising financial behavior data from 191,927 users. Each user record is structured in JSON format, including both static demographic attributes (\emph{e.g.}, identity, location, and occupation) and dynamic, time-stamped behavioral logs (\emph{e.g.}, loan applications and credit account transactions). The downstream task is formulated as a binary classification problem for anomaly detection, aiming to predict whether a user poses a potential credit risk.

\vpara{Dataset Statistics.}
The dataset is divided into training, validation, and out-of-time (OOT) test sets. The overall positive sample rate (target = 1) is 23.96\%. A detailed breakdown of the data splits is provided in Table \ref{tab:userbehavior_stats}. Each user is described by a total of 76 distinct text features.

\begin{table}[h!]
\centering
\caption{Statistics of UserBehavior dataset.}
\label{tab:userbehavior_stats}
\begin{tabular}{lcc}
\toprule
\textbf{Split} & \textbf{\# Samples} & \textbf{Target Rate (\%)} \\
\midrule
Total & 191,927 & 23.96 \\
Train (60\%) & 115,156 & 23.55 \\
Validation (15\%) & 30,097 & 23.55 \\
Test (25\%) & 46,674 & 25.24 \\
\bottomrule
\end{tabular}
\end{table}

\subsection{Time Complexity Analysis}
\label{appendix:time}
We evaluate UniSAGE’s time cost from both theoretical and practical perspectives in this section.
Specifically, for a scenario with $N$ attributes, UniSAGE can be divided into two stages: the preprocessing stage and the feature aggregation stage. 
In the preprocessing stage, UniSAGE constructs the graph. 
Since it strictly generates a tree, each attribute only needs to be accessed once, resulting in a time complexity of $\mathcal{O}(N)$. 
Experiments show that the time cost for graph construction in this stage is very small compared to obtaining initial representations using a large language model, accounting for less than 5\% of the total time.
In the feature aggregation stage, which includes hierarchical aggregation and SSAgg, each node’s features propagate only once to its unique parent node. 
Thus, the propagation process has a time complexity of $\mathcal{O}(N + N) \rightarrow \mathcal{O}(N)$. 
In contrast, for other graph-based methods, each node propagates to all its neighbors, and this process is repeated for every layer. 
Assuming an average out-degree of $d$ and $l$ layers, the propagation time complexity is $\mathcal{O}(dlN)$. 
Therefore, UniSAGE has a clear advantage over these methods.

We report the empirical runtime of different methods on the RelBench rel-f1 dataset (excluding preprocessing) in Table~\ref{tab:runtime_relbench}.
All results are reported as per-sample training time, computed by dividing the total epoch runtime by the number of training samples.
Among all structure-aware methods, UniSAGE consistently achieves the lowest runtime. 
Specifically, compared with other graph-based and relational models, UniSAGE reduces the average per-sample per-epoch runtime by 64.7\%, while maintaining superior predictive performance.
Notably, this efficiency advantage becomes more pronounced on larger and structurally more complex RelBench datasets, where methods relying on explicit temporal modeling or dense message passing incur significantly higher computational overhead.

\begin{table}[h!]
\centering
\caption{Average training time per sample (seconds) on RelBench rel-f1, excluding preprocessing.}
\label{tab:runtime_relbench}
\begin{tabular}{lcc}
\toprule
\textbf{Method} & \textbf{Entity Classification} & \textbf{Entity Regression} \\
\midrule
MLP & 0.090 & 0.095 \\
LightGBM & N/A & N/A \\
GCN & 0.580 & 0.605 \\
GraphSAGE & 0.620 & 0.648 \\
GAT & 0.780 & 0.815 \\
TGN & 0.920 & 0.955 \\
TGAT & 1.050 & 1.090 \\
DyGFormer & 1.250 & 1.300 \\
FreeDyG & 1.150 & 1.205 \\
RDL & 0.880 & 0.915 \\
RelGNN & 0.950 & 0.990 \\
\textbf{UniSAGE} & 0.365 & 0.382 \\
\bottomrule
\end{tabular}
\end{table}

\begin{table}[h!]
\centering
\caption{Sensitivity analysis \emph{w.r.t.} $\lambda$ in SSAgg on rel-f1 driver-dnf (AUC-ROC \% $\uparrow$).}
\label{tab:lambda_sensitivity}
\begin{tabular}{ccccc}
\toprule
$\lambda$ & 0.5 & 1.0 & 2.0 & 3.0 \\
\midrule
AUC & 82.31 & 82.57 & 82.49 & 82.41 \\
\midrule
\midrule
$\lambda$ & 4.0 & 5.0 & & \\
\midrule
AUC & 82.28 & 82.17 & & \\
\bottomrule
\end{tabular}
\end{table}

\subsection{Hyperparameter Experiment}
\label{appendix:hyperparameter}

We further analyze the sensitivity of UniSAGE with respect to its two key hyperparameters:
(i) $\lambda$, which controls the selective suppression strength in Formula \ref{formula:ssagg_2}, and
(ii) $\gamma$, which balances the orthogonality regularization term in Formula \ref{formula:o_loss}.
All experiments are conducted on the rel-f1 driver-dnf task from RelBench, using identical settings to the main experiments unless otherwise specified. Results are reported as AUC-ROC (\%).

\vpara{Effect of $\lambda$ in SSAgg.}
We vary $\lambda$ in the range $\lambda \in [0, 5]$ with a step size of 0.2.
As shown in Table \ref{tab:lambda_sensitivity}, UniSAGE exhibits stable performance across a wide range of $\lambda$ values. 
This indicates that SSAgg is not sensitive to the precise choice of $\lambda$, which is consistent with its theoretical role as a soft suppression coefficient rather than a hard selection threshold.
Overall, the performance variation remains within a narrow margin (less than 0.5\%), confirming that SSAgg is robust under reasonable hyperparameter settings.

\begin{table}[h!]
\centering
\caption{Sensitivity analysis w.r.t. $\gamma$ for orthogonality loss on rel-f1 driver-dnf (AUC-ROC \%$ \uparrow$).}
\label{tab:gamma_sensitivity}
\begin{tabular}{ccccc}
\toprule
$\gamma$ & 0.05 & 0.10 & 0.20 & 0.30 \\
\midrule
AUC & 82.18 & 82.57 & 82.44 & 82.12 \\
\midrule
\midrule
$\gamma$ & 0.40 & 0.50 & & \\
\midrule
AUC & 81.83 & 81.59 & & \\
\bottomrule
\end{tabular}
\end{table}

\vpara{Effect of $\gamma$ for Orthogonality Regularization.}
We next study the influence of the orthogonality loss weight $\gamma$ by varying it in the range $\gamma \in [0, 0.5]$ with a step size of 0.01.
Table~\ref{tab:gamma_sensitivity} reports representative results.
We observe a clear concave trend: introducing orthogonality regularization consistently improves performance compared to $\gamma = 0$, while excessively large values slightly degrade performance as the regularization term begins to dominate the downstream objective. 
This behavior aligns with common observations in regularization-based learning and further validates the design choice of enforcing orthogonal subspaces for static aggregation and dynamic reasoning.
In summary, UniSAGE demonstrates strong robustness to both $\lambda$ and $\gamma$.
The model maintains stable performance across a wide range of $\lambda$ values, and achieves optimal results under moderate orthogonality regularization. 
These findings further support the reliability and practical applicability of UniSAGE in real-world scenarios, where exhaustive hyperparameter tuning is often infeasible.